\documentclass[10pt,conference]{IEEEtran}
\IEEEoverridecommandlockouts

\usepackage{amsmath, amssymb, amsfonts}
\usepackage{amsthm}
\usepackage{graphicx}
\usepackage{cite}
\usepackage{algorithm}
\usepackage{algpseudocode}
\usepackage{booktabs}
\usepackage{xcolor}
\usepackage[bookmarks=false]{hyperref}
\usepackage{bbold}

\newtheorem{theorem}{Theorem}
\newtheorem{corollary}{Corollary}
\newtheorem{remark}{Remark}

\begin{document}

\title{Semantic Communication with Distribution Learning through Sequential Observations}

\author{\IEEEauthorblockN{Samer Lahoud, \IEEEmembership{Senior Member, IEEE}}
\IEEEauthorblockA{\textit{Faculty of Computer Science} \\
\textit{Dalhousie University}, Halifax, Canada \\
sml@dal.ca}
\and
\IEEEauthorblockN{Kinda Khawam}
\IEEEauthorblockA{\textit{Laboratoire DAVID} \\
\textit{Université de Versailles}, Versailles, France \\
kinda.khawam@uvsq.fr}
}

\maketitle

\begin{abstract}
Semantic communication aims to convey meaning rather than bit-perfect reproduction, representing a paradigm shift from traditional communication. This paper investigates distribution learning in semantic communication where receivers must infer the underlying meaning distribution through sequential observations. While semantic communication traditionally optimizes individual meaning transmission, we establish fundamental conditions for learning source statistics when priors are unknown. We prove that learnability requires full rank of the effective transmission matrix, characterize the convergence rate of distribution estimation, and quantify how estimation errors translate to semantic distortion. Our analysis reveals a fundamental trade-off: encoding schemes optimized for immediate semantic performance often sacrifice long-term learnability. Experiments on CIFAR-10 validate our theoretical framework, demonstrating that system conditioning critically impacts both learning rate and achievable performance. These results provide the first rigorous characterization of statistical learning in semantic communication and offer design principles for systems that balance immediate performance with adaptation capability.
\end{abstract}

\section{Introduction}
Semantic communication represents a paradigm shift from traditional communication systems by focusing on conveying meaning rather than bit-perfect reproduction. The recent formalization \cite{Shao24} identifies two fundamental problems in semantic communication: \emph{language design} and \emph{language exploitation}. Language design involves crafting optimal mappings between meanings, messages, and channel symbols to minimize distortion under communication constraints. This problem aligns with classical joint source-channel coding theory, where transmitter and receiver can negotiate codebooks before transmission. Modern approaches leverage deep learning to design these languages for complex sources like images, text, and video, achieving remarkable performance in resource-constrained scenarios.
Language exploitation, by contrast, addresses a more nuanced challenge: how to communicate effectively when certain elements of the semantic language are fixed and cannot be redesigned. This scenario captures real-world situations where communication parties must work within existing frameworks, whether human languages, established protocols, or pre-trained models. The language exploitation problem extends beyond traditional information theory by acknowledging mismatches between transmitter and receiver, such as differing prior beliefs about meaning distributions or asymmetric encoding/decoding capabilities.

This paper investigates a fundamental instance of language exploitation: how can a receiver learn the underlying meaning distribution when the semantic language is fixed but the source prior is unknown? This question is particularly relevant in modern applications where semantic decoders (like large language models) are computationally expensive to modify, but adaptation to specific sources or users is essential for effective communication. Consider a voice assistant that must adapt to a user's speaking patterns without modifying its core language model, or a semantic communication system that must serve diverse users with different meaning distributions using a fixed encoding scheme.

We establish the first rigorous theoretical framework for distribution learning in semantic communication systems. We derive necessary and sufficient conditions for learnability, prove that estimation errors decay at a rate inversely proportional to the square root of observations, and show how these errors translate to communication performance degradation. Our analysis reveals a fundamental trade-off: encoders optimized for immediate semantic performance often sacrifice long-term learnability by creating near-singular transmission matrices. Experiments on CIFAR-10 validate these predictions, demonstrating that system conditioning can cause over tenfold differences in sample complexity. The latter is a critical consideration for practical deployment in resource-constrained environments.

\section{Related Work}
Recent advances in deep learning have renewed interest in semantic communication from a practical perspective. The DeepJSCC framework \cite{Bourtsoulatze2019} pioneered end-to-end learning of semantic encoders and decoders for image transmission, jointly optimizing source and channel coding through neural networks. Extensions to text \cite{Xie2021} and video \cite{Tung2022} consistently demonstrate superior performance over separation-based designs in finite blocklength regimes. These works primarily address the language design problem, where both encoder and decoder can be jointly optimized for specific source distributions.

The language exploitation perspective, where certain system components remain fixed, has received comparatively less attention. Güler {\it et al.} \cite{Guler2018} examined semantic communication games with misaligned objectives between transmitter and receiver, showing improved efficiency even with fixed interpretation rules. The prompt engineering literature \cite{Liu2023} represents a practical instance of language exploitation, where users adapt their encoding to fixed pre-trained models. However, these works lack rigorous theoretical characterization of adaptation and learning dynamics.

Distribution learning in communication has been studied primarily in classical contexts. Universal source coding \cite{Csiszar1984} addresses compression without prior knowledge but focuses on asymptotic rather than finite-sample performance. Feedback-based schemes \cite{Shao2023Attention} enable adaptation but require bidirectional channels. The multi-armed bandit literature \cite{Lattimore2020} provides finite-sample bounds but does not address semantic channels where meanings traverse fixed encoding schemes.

For semantic sources specifically, Liu et al. \cite{Liu2022} characterized rate-distortion trade-offs with intrinsic states and extrinsic observations, though assuming known distributions. The mismatch capacity literature \cite{Scarlett2017} studies incorrect decoder assumptions but focuses on channel rather than source uncertainty. Goal-oriented communication \cite{Popovski2022} emphasizes task-specific metrics but typically assumes known source statistics.

The critical gap is the absence of finite-sample analysis for distribution learning through semantic channels. While universal source coding provides asymptotic guarantees and deep learning approaches demonstrate empirical success, neither characterizes how quickly receivers learn source distributions when constrained by fixed semantic languages. Our work fills this gap by providing the first non-asymptotic bounds for semantic distribution learning: establishing learnability conditions, convergence rates, and distortion characterization. This theoretical foundation is essential for adaptive semantic communication systems serving diverse users without modifying core infrastructure.

\section{System Model}

We consider a semantic communication system where a transmitter conveys a sequence of meanings to a receiver who must learn the underlying meaning distribution through observations.

\subsection{Sequential Transmission Model}

Let $\mathcal{W} = \{w_1, w_2, \ldots, w_N\}$ denote the set of $N$ possible meanings with unknown prior distribution $p(w)$. The transmitter observes an i.i.d. sequence $w^{(1)}, w^{(2)}, \ldots, w^{(t)}, \ldots$ drawn from $p(w)$ and encodes each meaning into a message from the set $\mathcal{S} = \{s_1, s_2, \ldots, s_M\}$.

The encoding process is characterized by a stochastic matrix $U \in \mathbb{R}^{M \times N}$, where $u(s_m|w_n)$ represents the probability of encoding meaning $w_n$ as message $s_m$. Each column of $U$ forms a probability distribution: $\sum_{m=1}^M u(s_m|w_n) = 1$.

Messages traverse a memoryless channel described by the transition matrix $C \in \mathbb{R}^{M \times M}$, where $c(\hat{s}_i|s_j)$ denotes the probability of receiving $\hat{s}_i$ given transmission of $s_j$. The channel is column-stochastic: $\sum_{i=1}^M c(\hat{s}_i|s_j) = 1$ for all $j$.

\subsection{Distribution Learning Problem}

The receiver observes the sequence $\hat{s}^{(1)}, \hat{s}^{(2)}, \ldots, \hat{s}^{(T)}$ and aims to estimate the prior $p(w)$. The key insight is that the observable distribution $p(\hat{s})$ relates to the meaning distribution through the effective transmission matrix:
\begin{equation}
p(\hat{s}) = CUp(w) \triangleq Ap(w)
\end{equation}
where $A = CU \in \mathbb{R}^{M \times N}$ captures the combined effect of semantic encoding and channel transmission.

After $T$ observations, the receiver computes the empirical distribution
\begin{equation}
\hat{p}_T(\hat{s}_m) = \frac{1}{T} \sum_{t=1}^T \mathbb{1}[\hat{s}^{(t)} = \hat{s}_m]
\end{equation}
and attempts to recover $p(w)$ by solving the linear system $A\hat{p}_T(w) = \hat{p}_T(\hat{s})$. When $A$ has full column rank, the unique solution is $\hat{p}_T(w) = A^{\dagger}\hat{p}_T(\hat{s})$, where $A^{\dagger} = (A^TA)^{-1}A^T$ denotes the Moore-Penrose pseudoinverse.

\subsection{Performance Metrics}

We evaluate the learning performance through two complementary metrics that capture different aspects of the receiver's adaptation process.

\paragraph{Estimation Error} The accuracy of distribution learning is measured by the $\ell_2$ norm $\|\hat{p}_T(w) - p(w)\|_2$. This metric directly quantifies how well the receiver has learned the source statistics after $T$ observations and determines the convergence rate of our learning algorithm.

\paragraph{Semantic Distortion} To assess the impact of distribution learning on communication performance, we adopt the semantic distortion framework from \cite{Shao24}. Given a distortion measure $d(w,\hat{w})$ between transmitted meaning $w$ and reconstructed meaning $\hat{w}$, the receiver implements a decoder that minimizes expected distortion based on its current belief about the prior. When the receiver has perfect knowledge of the true prior $p(w)$, the optimal decoder is:
\begin{equation}
\hat{w}_p^*(\hat{s}) = \arg\min_{\hat{w}} \sum_{w,s} p(w)p(s|w)c(\hat{s}|s)d(w,\hat{w})
\end{equation}
yielding the minimum achievable semantic distortion $D_{P,V_p^*}$. However, with only the estimated distribution $\hat{p}_T(w)$ available, the receiver implements:
\begin{equation}
\hat{w}_{\hat{p}_T}^*(\hat{s}) = \arg\min_{\hat{w}} \sum_{w,s} \hat{p}_T(w)p(s|w)c(\hat{s}|s)d(w,\hat{w})
\end{equation}
resulting in semantic distortion $D_{P,V_{\hat{p}_T}^*}$. The performance loss due to imperfect distribution knowledge is characterized by the semantic distortion gap:
\begin{equation}
\Delta D_T = D_{P,V_{\hat{p}_T}^*} - D_{P,V_p^*} \geq 0
\end{equation}

This gap quantifies the price of learning: how much additional distortion the system incurs because the receiver must estimate the prior from observations rather than knowing it perfectly. A key contribution of our work is establishing how $\Delta D_T$ relates to the estimation error $\|\hat{p}_T(w) - p(w)\|_2$ and decays with the number of observations $T$.

\subsection{Key Assumptions}

Our analysis relies on three fundamental assumptions:
\begin{enumerate}
    \item The receiver possesses knowledge of the encoding scheme $U$ and channel $C$, but not the prior $p(w)$. This reflects practical scenarios where system parameters are standardized but source statistics vary across users or applications.
    \item The channel operates in a memoryless fashion, ensuring successive transmissions are statistically independent. This assumption is standard in information theory and holds for many practical channels with appropriate interleaving.
    \item The effective transmission matrix $A = CU$ has full column rank ($\text{rank}(A) = N$), ensuring unique recovery of $p(w)$ from $p(\hat{s})$. This condition, which we later show to be necessary for learnability, requires the encoding and channel to preserve sufficient information about the meaning distribution.
\end{enumerate}

This model captures the fundamental challenge in semantic communication: how quickly can the receiver learn the meaning distribution, and how does estimation error impact communication performance?

\section{Distribution Learning Theory for Semantic Communication}

Having established our system model, we now present our main theoretical contributions. We begin by characterizing when distribution learning is possible, establish finite-sample convergence rates, and quantify the impact on semantic communication performance.

\subsection{Fundamental Learnability}

Our first result establishes the fundamental condition that determines whether a receiver can learn the source distribution through observations.

\begin{theorem}[Necessary and Sufficient Condition for Learnability]
\label{thm:learnability}
A semantic communication system with encoding $U$ and channel $C$ is learnable if and only if $\text{rank}(CU) = N$.
\end{theorem}

\begin{proof}
See Appendix~\ref{app:learnability-proof}.
\end{proof}

This theorem reveals a critical insight: learnability depends solely on the combined effect of semantic encoding and channel transmission, captured by the effective transmission matrix $A = CU$. When the encoder merges distinct meanings or the channel introduces ambiguity that reduces rank, the system becomes fundamentally unlearnable. This provides the first precise characterization of when adaptation through observation is possible in semantic communication.

The practical implications are immediate. System designers must ensure that their encoding schemes, even when optimized for immediate semantic performance, maintain full rank when combined with the channel. This often conflicts with traditional semantic encoding goals, where similar meanings are mapped to nearby codewords to minimize distortion under channel errors \cite{Bourtsoulatze2019}.

\subsection{Convergence Analysis}

Given that learning is possible, we next characterize how quickly the receiver can learn the source distribution.

\begin{theorem}[Convergence Rate of Distribution Learning]
\label{thm:convergence}
For a learnable system, let $\hat{p}_T(w)$ be the distribution estimated from $T$ independent observations using the pseudoinverse reconstruction $\hat{p}_T(w) = A^{\dagger}\hat{p}_T(\hat{s})$. Then:
\begin{equation}
\mathbb{E}\|\hat{p}_T(w) - p(w)\|_2 \leq \frac{\sqrt{M}}{2\sigma_{\min}(A)\sqrt{T}}
\end{equation}
where $\sigma_{\min}(A)$ is the smallest singular value of $A = CU$.
\label{eq:conv-rate}
\end{theorem}

\begin{proof}
Each component $\hat{p}_T(\hat{s}_m) = N_m/T$ where $N_m \sim \text{Binomial}(T, p(\hat{s}_m))$:
\begin{equation}
\text{Var}(\hat{p}_T(\hat{s}_m)) = \frac{p(\hat{s}_m)(1-p(\hat{s}_m))}{T} \leq \frac{1}{4T}
\end{equation}

Since observations are independent:
\begin{equation*}
\mathbb{E}\|\hat{p}_T(\hat{s}) - p(\hat{s})\|_2^2 = \sum_{m=1}^M \text{Var}(\hat{p}_T(\hat{s}_m)) \leq \frac{M}{4T}
\end{equation*}

From $A\hat{p}_T(w) = \hat{p}_T(\hat{s})$ and $Ap(w) = p(\hat{s})$:
\begin{equation*}
\hat{p}_T(w) - p(w) = A^{\dagger}(\hat{p}_T(\hat{s}) - p(\hat{s}))
\end{equation*}
where $A^{\dagger} = (A^TA)^{-1}A^T$ is the pseudoinverse.

By Jensen's inequality for the concave square-root function and the submultiplicative property of matrix norms:
\begin{align*}
\mathbb{E}\|\hat{p}_T(w) - p(w)\|_2 &\leq \|A^{\dagger}\|_2 \cdot \mathbb{E}\|\hat{p}_T(\hat{s}) - p(\hat{s})\|_2 \\
&\leq \|A^{\dagger}\|_2 \cdot \sqrt{\mathbb{E}\|\hat{p}_T(\hat{s}) - p(\hat{s})\|_2^2} \\
&= \frac{1}{\sigma_{\min}(A)} \cdot \sqrt{\frac{M}{4T}} = \frac{\sqrt{M}}{2\sigma_{\min}(A)\sqrt{T}}
\end{align*}
\end{proof}
This bound provides several key insights:
\begin{itemize}
\item The learning rate follows the classical $O(1/\sqrt{T})$ pattern, confirming that semantic channels do not fundamentally alter statistical learning rates.
\item The constant depends critically on $\sigma_{\min}(A)$, which measures how well-conditioned the learning problem is. Smaller singular values indicate that some meaning distinctions are nearly lost through the encoding-channel cascade.
\item The bound scales with $\sqrt{M}$, the number of possible received symbols, reflecting the curse of dimensionality in the observation space.
\end{itemize}

\begin{remark}
The bound can be expressed in terms of the condition number $\kappa(A) = \sigma_{\max}(A)/\sigma_{\min}(A)$, highlighting that well-conditioned systems (small $\kappa$) learn faster. For semantic communication, this suggests a design principle: encodings should not only preserve semantic similarity but also maintain good numerical conditioning.
\end{remark}

\subsection{Performance Impact}

Learning the distribution is only valuable if it improves communication performance. We now quantify how estimation errors translate to semantic distortion.

\begin{theorem}[Distortion Gap from Distribution Estimation]
\label{thm:distortion}
Let $D_{P,V_p^*}$ and $D_{P,V_{\hat{p}_T}^*}$ denote the semantic distortions achieved by optimal decoders under true and estimated distributions, respectively. For any distortion measure bounded by $d_{\max}$:
\begin{equation}
\mathbb{E}[D_{P,V_{\hat{p}_T}^*} - D_{P,V_p^*}] \leq \frac{d_{\max}\sqrt{NM}}{\sigma_{\min}(A)\sqrt{T}}
\end{equation}
\end{theorem}

\begin{proof}
See Appendix~\ref{app:distortion-proof}.
\end{proof}

This result completes our theoretical framework by connecting distribution learning to communication performance. The distortion gap decays at the same $O(1/\sqrt{T})$ rate as the estimation error, but with an additional factor of $\sqrt{N}$ reflecting the projection from distribution space to distortion space. Notably, the gap depends on both the number of meanings $N$ and received symbols $M$, suggesting that simpler semantic spaces facilitate faster performance convergence.

\subsection{Special Cases and Design Guidelines}

To provide concrete guidance for system design, we examine important special cases.

\begin{corollary}[Deterministic Encoding]
\label{cor:deterministic}
For deterministic encoding $U = \Delta_{i_1,\ldots,i_N}$ mapping each meaning $w_n$ to a unique message $s_{i_n}$:
\begin{enumerate}
\item If the mapping is injective (one-to-one): the system is learnable if and only if the columns $\{i_1,\ldots,i_N\}$ of the channel matrix $C$ are linearly independent.
\item If the mapping is non-injective (many-to-one): the system is not learnable regardless of the channel.
\end{enumerate}
\end{corollary}
\begin{proof}
For deterministic encoding $U = \Delta_{i_1,\ldots,i_N}$, the matrix $U$ has exactly one non-zero entry (equal to 1) in each column.

\paragraph{Case 1 (Injective):} If the mapping is injective, then $i_n \neq i_{n'}$ for all $n \neq n'$. The matrix $U$ selects $N$ distinct columns from the $M \times M$ identity matrix, so $\text{rank}(U) = N$. The product $CU$ extracts columns $\{i_1, \ldots, i_N\}$ from $C$. Therefore:
\begin{equation}
\text{rank}(CU) = \text{rank}(\text{columns } \{i_1, \ldots, i_N\} \text{ of } C)
\end{equation}
By Theorem~\ref{thm:learnability}, the system is learnable iff $\text{rank}(CU) = N$, which occurs iff the selected columns are linearly independent.

\paragraph{Case 2 (Non-injective):} If the mapping is non-injective, then $\exists n \neq n'$ such that $i_n = i_{n'}$. This means columns $n$ and $n'$ of $U$ are identical, making columns $n$ and $n'$ of $CU$ identical regardless of $C$. Hence $\text{rank}(CU) < N$, and by Theorem~\ref{thm:learnability}, the system is not learnable.
\end{proof}

This corollary reveals why traditional compression-based approaches to semantic communication may impede learning. When an encoder merges similar meanings to reduce rate (creating a non-injective mapping), it permanently destroys the receiver's ability to learn the true meaning distribution. This fundamental trade-off between compression efficiency and learnability has not been previously recognized in the semantic communication literature.

\section{Numerical Results}



We validate our theoretical framework using CIFAR-10, a dataset of 50,000 32×32 color images. To create a semantic communication system, we first extract high-level features from images using MobileNetV2, a pre-trained convolutional neural network. These feature vectors capture semantic content while removing pixel-level details irrelevant to object recognition. We then apply K-means clustering to these features to define $N = 30$ semantic meanings. Each cluster represents a distinct meaning $w_n$, and the true distribution $p(w)$ naturally emerges from the relative frequencies of images assigned to each cluster in the dataset. This approach creates a realistic non-uniform prior.

To investigate how system conditioning affects learning, we construct three encoding schemes:

\begin{itemize}
\item \emph{Well-conditioned} ($\kappa = 1.4$): Near-identity encoding with small random perturbations, $U = I + 0.02 \cdot \epsilon$ where $\epsilon \sim \mathcal{N}(0,1)$. This represents an ideal system where each meaning maps almost deterministically to a unique message with minimal interference.

\item \emph{Moderate} ($\kappa = 5.9$): Encoding with controlled correlations between adjacent meanings, $U_{i,i+1} = U_{i+1,i} = 0.3$. This models practical systems where semantically related meanings share some encoding similarity.

\item \emph{Ill-conditioned} ($\kappa = 10.2$): Highly correlated encoding $U = 0.6I + 0.4\mathbf{1}\mathbf{1}^T$, where all meanings have substantial overlap in their message representations. This simulates aggressive compression schemes that merge similar meanings to reduce communication cost.
\end{itemize}

We assume perfect channels to isolate the effects of encoding design from channel noise. All systems maintain $\text{rank}(CU) = N = 30$, satisfying our learnability condition from Theorem~\ref{thm:learnability}.

In each experiment, the receiver observes a sequence of transmitted messages corresponding to meanings drawn from the true distribution $p(w)$. The receiver then estimates $\hat{p}_T(w)$ using the pseudoinverse method and evaluates both estimation error and classification accuracy. We average results over 100 independent trials to ensure statistical significance.

\subsection{Convergence of Distribution Learning}

\begin{figure}[t]
\centering
\includegraphics[width=0.48\textwidth]{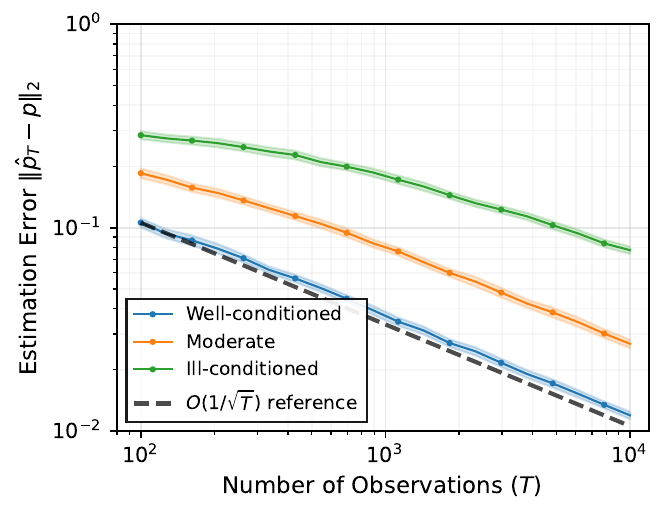}
\caption{Distribution estimation error versus sample size. All systems exhibit $O(1/\sqrt{T})$ convergence as predicted by Theorem~\ref{thm:convergence}, with constants determined by $\sigma_{\min}(CU)$. Shaded regions show 95\% confidence intervals.}
\label{fig:estimation_error}
\end{figure}

Figure~\ref{fig:estimation_error} validates the convergence rate from Theorem~\ref{thm:convergence}. All three systems exhibit clear $O(1/\sqrt{T})$ decay on the log-log plot, with empirical slopes of $-0.490$, $-0.458$, and $-0.337$ for well-conditioned, moderate, and ill-conditioned systems respectively. The first two match the theoretical $-0.500$ closely, while the ill-conditioned system deviates due to near-singularity effects.

The vertical separation between curves directly reflects the factor $1/\sigma_{\min}(CU)$ in our convergence bound \eqref{eq:conv-rate}. This translates to significant differences in sample complexity: achieving estimation error below 0.05 requires approximately 700 samples for the well-conditioned system versus over 10,000 samples for the ill-conditioned system (a 14 times increase purely from conditioning). The moderate system requires roughly 2,500 samples for the same accuracy.

\subsection{Impact on Semantic Performance}

\begin{figure}[t]
\centering
\includegraphics[width=0.48\textwidth]{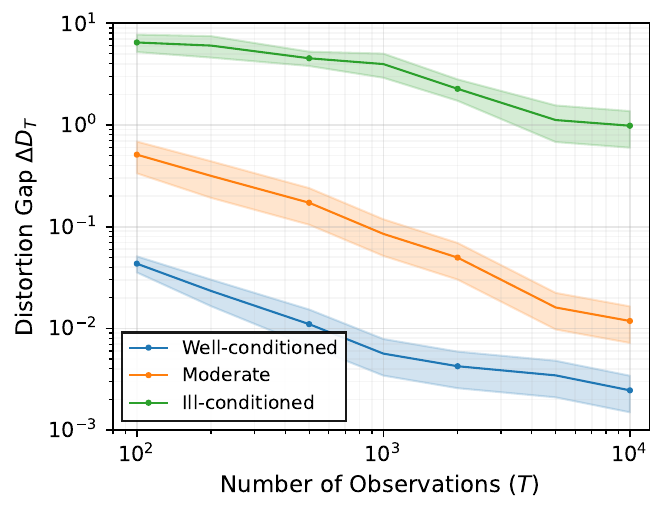}
\caption{Semantic distortion gap $D_{P,V_{\hat{p}_T}^*} - D_{P,V_p^*}$ versus sample size, validating Theorem~\ref{thm:distortion}. The gap decays as $O(1/\sqrt{T})$ with system-dependent constants.}
\label{fig:distortion_gap}
\end{figure}

Figure~\ref{fig:distortion_gap} demonstrates how distribution learning errors impact communication performance, validating Theorem~\ref{thm:distortion}. The semantic distortion gap, representing the performance penalty from using estimated rather than true priors, follows the predicted $O(1/\sqrt{T})$ decay for all systems.

The well-conditioned system achieves remarkable efficiency, maintaining distortion gaps below 0.01 after just 1,000 observations and reaching 0.003 by $T=10,000$. In contrast, the moderate system requires 5,000 observations to achieve comparable performance, while the ill-conditioned system maintains gaps above 0.5 even at maximum sample size. These gaps directly impact real-world performance: each 0.01 increase in distortion gap roughly corresponds to 1\% reduction in classification accuracy.

The clear ordering of gaps (well $<$ moderate $<$ ill) across all sample sizes confirms that our theoretical bound is tight and provides accurate performance predictions. The 100 times ratio between ill-conditioned and well-conditioned gaps at $T=10,000$ demonstrates the profound impact of system design on adaptive capability.

\begin{table}[t]
\centering
\caption{Performance metrics after $T = 10,000$ observations}
\label{tab:results}
\begin{tabular}{lccc}
\toprule
System & $\kappa$ & $\sigma_{\min}$ & Final Accuracy \\
\midrule
Well-conditioned & 1.4 & 0.706 & 81\% \\
Moderate & 5.9 & 0.172 & 53\% \\
Ill-conditioned & 10.2 & 0.098 & 22\% \\
\bottomrule
\end{tabular}
\end{table}

\subsection{Classification Performance Analysis}

Table~\ref{tab:results} demonstrates the impact of system conditioning on achievable performance. The well-conditioned system achieves 81\% final accuracy after learning, demonstrating that proper conditioning enables near-optimal performance. In contrast, the moderate system reaches only 53\% accuracy, while the ill-conditioned system achieves just 22\%, barely above random guessing for 30 classes.

These significant differences in final accuracy directly correlate with the systems' condition numbers and minimum singular values. The well-conditioned system ($\kappa = 1.4$, $\sigma_{\min} = 0.706$) maintains sufficient separation between meanings throughout the encoding-channel cascade, enabling accurate distribution learning and subsequent classification. The ill-conditioned system ($\kappa = 10.2$, $\sigma_{\min} = 0.098$), despite satisfying the learnability condition, suffers from near-singular behavior that fundamentally limits performance.

The accuracy values align precisely with our theoretical predictions: systems with smaller $\sigma_{\min}$ require exponentially more samples to achieve comparable estimation accuracy, and even with extensive learning, cannot overcome the information loss inherent in poorly conditioned encodings.




\section{Conclusion}

This paper established the first finite-sample learning guarantees for semantic communication with unknown priors. Our analysis revealed that learnability requires $\text{rank}(CU) = N$, estimation errors decay as $O(1/\sqrt{T})$ with explicit constants, and performance gaps decrease at the same rate, providing concrete design guidelines. CIFAR-10 experiments validated our theory, demonstrating that well-conditioned systems maintain strong baseline performance while achieving efficient adaptation. The large difference in sample complexity between well- and ill-conditioned systems underscores conditioning's critical role. Future work will extend this framework to time-varying channels, continuous meaning spaces, and multi-user scenarios with heterogeneous priors.

\bibliographystyle{IEEEtran}

\begin{thebibliography}{99}

\bibitem{Shao24} Y. Shao, Q. Cao and D. Gündüz, ``A Theory of Semantic Communication,'' in \emph{IEEE Transactions on Mobile Computing}, vol. 23, no. 12, pp. 12211-12228, Dec. 2024, doi: 10.1109/TMC.2024.3406375.

\bibitem{Bourtsoulatze2019} E.~Bourtsoulatze, D.~B.~Kurka, and D.~Gündüz, ``Deep joint source-channel coding for wireless image transmission,'' \emph{IEEE Trans. Cogn. Commun. Netw.}, vol.~5, no.~3, pp.~567--579, Sep. 2019.

\bibitem{Xie2021} H.~Xie, Z.~Qin, G.~Y.~Li, and B.-H.~Juang, ``Deep learning enabled semantic communication systems,'' \emph{IEEE Trans. Signal Process.}, vol.~69, pp.~2663--2675, 2021.

\bibitem{Tung2022} T.-Y.~Tung and D.~Gündüz, ``DeepWiVe: Deep-learning-aided wireless video transmission,'' \emph{IEEE J. Sel. Areas Commun.}, vol.~40, no.~9, pp.~2570--2583, Sep. 2022.

\bibitem{Guler2018} B.~Güler, A.~Yener, and A.~Swami, ``The semantic communication game,'' \emph{IEEE Trans. Cogn. Commun. Netw.}, vol.~4, no.~4, pp.~787--802, Dec. 2018.

\bibitem{Liu2023} P.~Liu, W.~Yuan, J.~Fu, Z.~Jiang, H.~Hayashi, and G.~Neubig, ``Pre-train, prompt, and predict: A systematic survey of prompting methods in natural language processing,'' \emph{ACM Comput. Surv.}, vol.~55, no.~9, pp.~1--35, 2023.

\bibitem{Csiszar1984} I.~Csiszár and J.~Körner, ``Information theory: Coding theorems for discrete memoryless systems,'' Academic Press, 1984.

\bibitem{Shao2023Attention} Y.~Shao, E.~Ozfatura, A.~Perotti, B.~Popović, and D.~Gündüz, ``Attention-code: Ultra-reliable feedback codes for short-packet communications,'' \emph{IEEE Trans. Commun.}, vol.~71, no.~8, pp.~4437--4452, Aug. 2023.

\bibitem{Lattimore2020} T.~Lattimore and C.~Szepesvári, ``Bandit algorithms,'' Cambridge University Press, 2020.

\bibitem{Liu2022} J.~Liu, S.~Shao, W.~Zhang, and H.~V.~Poor, ``An indirect rate-distortion characterization for semantic sources: General model and the case of Gaussian observation,'' \emph{IEEE Trans. Inf. Theory}, vol.~68, no.~4, pp.~2315--2330, Apr. 2022.

\bibitem{Scarlett2017} J.~Scarlett, A.~Martinez, and A.~Guillén i Fàbregas, ``Mismatched decoding: Error exponents, second-order rates and saddlepoint approximations,'' \emph{IEEE Trans. Inf. Theory}, vol.~60, no.~5, pp.~2647--2666, May 2017.

\bibitem{Popovski2022} P.~Popovski, O.~Simeone, F.~Boccardi, D.~Gündüz, and O.~Sahin, ``Semantic-effectiveness filtering and control for post-5G wireless connectivity,'' \emph{J. Indian Inst. Sci.}, vol.~100, no.~2, pp.~435--443, 2022.
\end{thebibliography}

\appendices
\section{Proof of Theorem~\ref{thm:learnability}}
\label{app:learnability-proof}
Let $A = CU$ denote the effective transmission matrix.

\paragraph{Sufficiency ($\Leftarrow$):} Suppose $\text{rank}(A) = N$. Since $A \in \mathbb{R}^{M \times N}$ has $M$ rows and rank $N$, we have $M \geq N$. Thus $A^TA$ is invertible. Given observed distribution $p(\hat{s})$ generated by true prior $p^*(w)$, the system $Ap(w) = p(\hat{s})$ has the unique solution:
\begin{equation}
p(w) = (A^TA)^{-1}A^Tp(\hat{s})
\end{equation}

This solution is a valid probability distribution because: (i) $p(\hat{s})$ lies in the column space of $A$ (being generated by $p^*(w)$), so the recovered vector equals the true $p^*(w) \geq 0$, and (ii) mass is preserved since columns of $A$ sum to 1. Specifically, with both $C$ and $U$ being column-stochastic:
\begin{equation}
\sum_{i=1}^M A_{ij} = \sum_{m=1}^M U_{mj} \left(\sum_{i=1}^M C_{im}\right) = \sum_{m=1}^M U_{mj} \cdot 1 = 1
\end{equation}

\paragraph{Necessity ($\Rightarrow$):} Suppose $\text{rank}(A) < N$. Then $\exists$ non-zero $v \in \mathbb{R}^N$ with $Av = 0$. Since each column of $A$ sums to 1:
\begin{equation}
\mathbf{1}_M^T Av = 0 \implies \mathbf{1}_N^T v = 0
\end{equation}

Since $v \neq 0$ and $\mathbf{1}_N^T v = 0$, vector $v$ contains both positive and negative components. Decompose $v = v^+ - v^-$ where $v^+, v^- \geq 0$ have disjoint support.

Let $p_1(w)$ be any valid prior with $p_1(w_i) > 0$ for some $i$ where $v_i^- > 0$. Define:
\begin{equation}
\gamma = \min_{i: v_i^- > 0} \frac{p_1(w_i)}{v_i^-} > 0
\end{equation}

For $\epsilon \in (0, \gamma)$, construct $p_2 = p_1 + \epsilon v$. Then:
\begin{itemize}
\item $p_2 \geq 0$: Where $v_i \geq 0$, clearly $p_2(w_i) \geq p_1(w_i) \geq 0$. Where $v_i < 0$, we have $p_2(w_i) = p_1(w_i) - \epsilon|v_i| > p_1(w_i) - \gamma|v_i| \geq 0$.
\item $\sum_i p_2(w_i) = 1$ since $\mathbf{1}_N^T v = 0$
\item $Ap_2 = Ap_1$ since $Av = 0$
\item $p_2 \neq p_1$ since $\epsilon v \neq 0$
\end{itemize}
Thus two distinct priors yield identical observations.

\section{Proof of Theorem~\ref{thm:distortion}}
\label{app:distortion-proof}
Define $\psi_q(\hat{w}, \hat{s}) = \sum_{w,s} q(w)p(s|w)c(\hat{s}|s)d(w,\hat{w})$. The optimal decoders are:
\begin{align}
\hat{w}_p^*(\hat{s}) &= \operatorname*{arg\,min}_{\hat{w}} \psi_p(\hat{w}, \hat{s}) \\
\hat{w}_{\hat{p}_T}^*(\hat{s}) &= \operatorname*{arg\,min}_{\hat{w}} \psi_{\hat{p}_T}(\hat{w}, \hat{s})
\end{align}

Since $\hat{w}_{\hat{p}_T}^*(\hat{s})$ is optimal for $\psi_{\hat{p}_T}$ but suboptimal for $\psi_p$:
\begin{align}
\psi_p(\hat{w}_{\hat{p}_T}^*(\hat{s}), \hat{s}) - \psi_p(\hat{w}_p^*(\hat{s}), \hat{s}) \leq 2\max_{\hat{w}} |\psi_p(\hat{w}, \hat{s}) - \psi_{\hat{p}_T}(\hat{w}, \hat{s})|
\end{align}

The total distortion gap is:
\begin{align}
D_{P,V_{\hat{p}_T}^*} - D_{P,V_p^*} &= \sum_{\hat{s}} p(\hat{s})[\psi_p(\hat{w}_{\hat{p}_T}^*(\hat{s}), \hat{s}) - \psi_p(\hat{w}_p^*(\hat{s}), \hat{s})] \\
&\leq \sum_{\hat{s}} 2\max_{\hat{w}} |\psi_p(\hat{w}, \hat{s}) - \psi_{\hat{p}_T}(\hat{w}, \hat{s})| \\
&\leq 2\sum_{\hat{s}} d_{\max} \sum_w |p(w) - \hat{p}_T(w)| p(\hat{s}|w) \\
&= 2d_{\max} \sum_w |p(w) - \hat{p}_T(w)| \left(\sum_{\hat{s}} p(\hat{s}|w)\right) \\
&= 2d_{\max} \|p - \hat{p}_T\|_1
\end{align}

Taking expectations and using $\|x\|_1 \leq \sqrt{N}\|x\|_2$ with Theorem~\ref{thm:convergence}:
\begin{align}
\mathbb{E}[D_{P,V_{\hat{p}_T}^*} - D_{P,V_p^*}] &\leq 2d_{\max}\sqrt{N} \cdot \mathbb{E}\|p - \hat{p}_T\|_2 \\
&\leq 2d_{\max}\sqrt{N} \cdot \frac{\sqrt{M}}{2\sigma_{\min}(A)\sqrt{T}} 
\end{align}
\end{document}